\documentclass[conference, compsocconf]{IEEEtran}
\usepackage[pdftex]{graphicx}
\usepackage{epstopdf}
\usepackage{rotating}
\usepackage{amsmath}
\usepackage{amssymb}
\usepackage{graphicx}
\usepackage{algorithm}
\usepackage{algorithmic}
\usepackage{subfigure}
\usepackage{threeparttable}

\usepackage{amsmath}
\usepackage{amssymb}
\newtheorem{theorem}{Theorem}

\newtheorem{defn}{Definition}

\newtheorem{proposition}{Proposition}
\newtheorem{example}{Example}

\begin{document}
%
\title{Matroidal structure of generalized rough sets based on symmetric and transitive relations}


\author{\IEEEauthorblockN{Bin Yang and William Zhu}
\IEEEauthorblockA{Lab of Granular Computing, Zhangzhou Normal University, Zhangzhou 363000, China.\\
Email: yangbin$_{-}$0906tdcq@126.com, williamfengzhu@gmail.com
}}
\maketitle

\begin{abstract}
Rough sets are efficient for data pre-process in data mining. Lower and upper approximations are two core concepts of rough sets. This paper studies generalized rough sets based on symmetric and transitive relations from the operator-oriented view by matroidal approaches.
We firstly construct a matroidal structure of generalized rough sets based on symmetric and transitive relations, and provide an approach to study the matroid induced by a symmetric and transitive relation. Secondly, this paper establishes a close relationship between matroids and generalized rough sets. Approximation quality and roughness of generalized rough sets can be computed by the circuit of matroid theory. At last, a symmetric and transitive relation can be constructed by a matroid with some special properties.
\end{abstract}


  %
  %
\section{Introduction}\label{section: introduction}
Rough set theory was originally proposed by Pawlak~\cite{Pawlak91Rough,Pawlak82Rough} as a tool for dealing with the
vagueness and granularity in information systems. This theory can approximately characterize an arbitrary subset of a universe
by using two definable subsets called lower and upper approximation operators~\cite{ChenZhangYeungTsang06Rough}. Now, with the
fast development of rough sets in recent years, it has already been applied in fields such as knowledge discovery~\cite{GalvezDCG00AnApplication},
machine learning, decision analysis, process control, pattern recognition and many other areas.

The core concepts of generalized rough sets
are lower and upper approximations based on relations on a universe. There are mainly two approaches for development of the rough set theory, the constructive and axiomatic approaches. In constructive approaches, lower and upper approximation operators are constructed from the primitive notions, such as binary relations on a universe, partitions of a universe, neighborhood systems; while the axiomatic approach, which is appreciate for studying the structures of rough set algebras, takes the lower and upper approximation operators as primitive notions. By taking advantage of these two approaches the rough set theory has been combined with other mathematical theories such as fuzzy sets, algebraic theory, topology and matroids. Specifically, the establishment of matroidal structures of rough sets may be much helpful for some problems of rough sets such as attribute reduction  and axiomatic in rough sets.

The concept of matroid was
originally introduced by Whitney~\cite{H.WhitneyOn} in 1935 as a generalization of graph theory and linear algebra. The concepts and results of
matroids were widely used in other fields such as algorithm design, combinatorial optimization and integer programming. Especially, the matroids provide well-established platforms for greedy algorithms. Since matroids appear in different mathematical branches, we can give explanations of matroidal structures in different mathematical backgrounds, then matroidal approaches plays a crucial role in other mathematical branches such as graph theory, linear algebra and rough sets\cite{LiLiu12Matroidal}.

The remainder of this paper is organized as follows: In Section \ref{section1}, some basic concepts and properties related to binary relations, generalized rough sets and matroids are introduced. In Section \ref{section3}, a matroidal structure of generalized rough sets based on symmetric and transitive relations is constructed. Moreover, we also explore the properties of the matroid induced by a symmetric and transitive relation. In Section \ref{section4}, the lower and upper approximations of generalized rough sets based on symmetric and transitive relations are described by circuits of matroids and an approach to generate a symmetric and transitive relation by a matroid is provided. Finally, Section \ref{section5} concludes this paper.
\section{Background }\label{section1}
In this section, we review some fundamental definitions and results of generalized rough sets
and matroids.
\subsection{Relations on a set}
In this subsection, we present some definitions and properties of binary relations used in this
paper. For detailed descriptions and proofs of them, please refer to~\cite{RajagopalMason92Discrete}.
\begin{defn}(Binary relation~\cite{RajagopalMason92Discrete})\label{definition1}
Let $U$ be a set, $U\times U$ the product set of $U$ and $U$. Any subset $R$ of $U\times U$ is
called a binary relation on $U$. For any $(x$, $y)\in U\times U$, if $(x$, $y)\in R$, then we say $x$
has relation with $y$, and denote this relationship as $xRy$.

For any $x\in U$, we call the set $\{y\in U\mid xRy\}$ the successor neighborhood of $x$ in $R$ and denote
it as $r_{R}(x)$.
\end{defn}

Throughout this paper, a binary relation is simply called a relation and it is defined on a finite and nonempty set.
The relation and its properties play important roles in studying generalized rough sets.
\begin{defn}(Symmetric relations~\cite{RajagopalMason92Discrete})
\label{defn}
Let $R$ be a relation on $U$. If for any $x$, $y\in U$, $y\in r_{R}(x)\Rightarrow x\in r_{R}(y)$, then
we say $R$ is symmetric.
\end{defn}

\begin{defn}(Transitive relations~\cite{RajagopalMason92Discrete})
\label{definition3}
Let $R$ be a relation on $U$. If for any $x$, $y$, $z\in U$, $y\in r_{R}(x)$ and $z\in r_{R}(y)\Rightarrow z\in r_{R}(x)$, then
we say $R$ is transitive.
\end{defn}

\subsection{Generalized rough sets}
\label{section2}
In this subsection, we present some definitions and results of generalized rough sets used in this paper. In this paper, we mainly study the generalized rough sets based on symmetric
and transitive relations. For detailed descriptions about generalized rough sets, please refer to ~\cite{YangLi06TheMinimization,ZhuWang03Reduction,ZhuWang06Axiomatic,Zhu07Generalized,Zhu09RelationshipBetween}.

\begin{defn}(Approximation space~\cite{ZhuWang03Reduction})
\label{definition4}
Let $R$ be a relation on a universe $U$. We call the ordered pair $(U$, $R)$ an approximation space.
\end{defn}

\begin{defn}(Generalized rough set~\cite{ZhuWang03Reduction})
\label{definition5}
Let $R$ be a relation on a universe $U$. A pair of approximation
operators $\underline{R}$, $\overline{R}$: $2^{U}\rightarrow 2^{U}$, are defined as follows:
for all $X\in 2^{U}$,
\begin{center}
$\underline{R}(X)=\{x\in U\mid r_{R}(x)\subseteq X\}$,\\
$\overline{R}(X)=\{x\in U\mid r_{R}(x)\bigcap X\neq \emptyset\}$.
\end{center}
They are called the lower approximation operator and the upper approximation operator, respectively.
For all $X\in 2^{U}$, if $\underline{R}(X)\neq \overline{R}(X)$, then $X$ is called a $R-$generalized rough set.
Otherwise, $X$ is called a $R-$precise set.
\end{defn}

\begin{theorem}~\cite{Zhu09RelationshipBetween}
\label{theorem1}
Let $R$ be a relation on universe $U$. For all $X\in 2^{U}$,\\
(1) $R$ is symmetric $\Leftrightarrow X\subseteq \underline{R}(\overline{R}(X))\Leftrightarrow \overline{R}(\underline{R}(X))\subseteq X$.\\
(2) $R$ is transitive $\Leftrightarrow \underline{R}(X)\subseteq \underline{R}(\underline{R}(X))\Leftrightarrow \overline{R}(\overline{R}(X))\subseteq \overline{R}(X)$.
\end{theorem}

In fact, Theorem \ref{theorem1} reveals the relationships between the properties of relation and approximation operators. In other words,
whether a relation is symmetric(transitive) or not, we could give the answer through the approximation operators.

\begin{defn}(Approximation quality and roughness~\cite{Pawlak91Rough})
\label{definition6}
Let $R$ be a relation on universe $U$. For all $X\in 2^{U}$, the approximation quality $\alpha_{R}(X)$
and roughness $\rho_{R}(X)$ of $X$ can be defined as follows:
\begin{center}
$\alpha_{R}(X)=\frac{|\underline{R}(X)|}{|\overline{R}(X)|}$,\\
$\rho_{R}(X)=1-\alpha_{R}(X)$.
\end{center}
\end{defn}

\subsection{Matroids}
Matroid theory was established as a generalization, or a connection, of graph theory and linear algebra.
This theory was used to study abstract relations on a subset, and it uses both of these areas of mathematics
for its motivation, its basic examples, and its notation. With the rapid development in recent years, the matroid
theory has already become an effective mathematic tool to study other mathematic branches. In this subsection, we
present definitions, examples and results of matroids used in this paper.

\begin{defn}(Matroid~\cite{Lai01Matroid})
\label{definition7}
A matroid is an ordered pair $M=(U, \mathcal{I})$ consisting a finite set $U$, and a collection $\mathcal{I}$ of subsets
of $U$ with the following three properties:\\
(I1) $\emptyset\in \mathcal{I}$;\\
(I2) If $I\in \mathcal{I}$, and $I'\subseteq I$, then $I'\in\mathcal{I}$;\\
(I3) If $I_{1}$, $I_{2}\in \mathcal{I}$, and $|I_{1}|<|I_{2}|$, then there exists $e\in I_{2}-I_{1}$ such that $I_{1}\bigcup\{e\}\in\mathcal{I}$, where $|I|$ denotes the cardinality of $I$.\\
Any element of $\mathcal{I}$ is called an independent set.
\end{defn}

\begin{example}
\label{example1}
Let $U=\{a$, $b$, $c\}$,
$\mathcal{I}_{1}=\{\emptyset$, $\{a\}$, $\{b\}$, $\{c\}$, $\{a$, $c\}$, $\{b$, $c\}\}$ and
$\mathcal{I}_{2}=\{\emptyset$, $\{a\}$, $\{b\}$, $\{c\}$, $\{a$, $b\}$, $\{b$, $c\}\}$.
Clearly, $\mathcal{I}_{1}$ and $\mathcal{I}_{2}$ satisfy the independent set axioms of matroids.
Hence $(U$, $\mathcal{I}_{1})$ and $(U$, $\mathcal{I}_{2})$ are matroids, respectively.
\end{example}

\begin{defn}(Dependent set~\cite{Lai01Matroid})
\label{definition8}
Let $M=(U$, $\mathcal{I})$ be a matroid. For any $X\subseteq U$, if $X\notin \mathcal{I}$, then we say
$X$ is a dependent set of $M$.
\end{defn}

\begin{defn}(Circuit~\cite{Lai01Matroid})
\label{definition9}
Let $M=(U$, $\mathcal{I})$ be a matroid. A minimal dependent set in $M$ is called a
circuit of $M$, and we denote the family of all circuits of $M$ by $\mathcal{C}(M)$.
\end{defn}

\begin{example}(Continued from Example \ref{example1})
\label{example2}
The family of dependent sets of $(U$, $\mathcal{I}_{1})$ is $\{\{a$, $b\}$, $\{a$, $b$, $c\}\}$
and the family of dependent sets of $(U$, $\mathcal{I}_{2})$ is $\{\{a$, $c\}$, $\{a$, $b$, $c\}\}$.
Hence, the family of circuits of $(U$, $\mathcal{I}_{1})$ and $(U$, $\mathcal{I}_{2})$ are $\mathcal{C}_{1}=\{\{a$, $b\}\}$ and
$\mathcal{C}_{2}=\{\{a$, $c\}\}$, respectively.
\end{example}

\begin{theorem}(Circuit axiom~\cite{Lai01Matroid})
\label{theorem2}
Let $M=(U$, $\mathcal{I})$ be a matroid and $\mathcal{C}=\mathcal{C}(M)$ , then
$\mathcal{C}$ satisfies the following three properties:\\
(C1) $\emptyset\notin\mathcal{C}$;\\
(C2) If $C_{1}$, $C_{2}\in \mathcal{C}$ and $C_{1}\subseteq C_{2}$, then $C_{1}=C_{2}$;\\
(C3) If $C_{1}$, $C_{2}\in \mathcal{C}$, $C_{1}\neq C_{2}$ and $e\in C_{1}\bigcap C_{2}$, then there
exists $C_{3}\in \mathcal{C}$ such that $C_{3}\subseteq (C_{1}\bigcup C_{2})-\{e\}$.
\end{theorem}

\begin{theorem}~\cite{Lai01Matroid}
\label{theorem3}
Let $U$ be a nonempty and finite set and $\mathcal{C}$ a family of subsets of $U$.
If $\mathcal{C}$ satisfies (C1), (C2) and (C3) of Theorem \ref{theorem2}, then there exists
$M=(U$, $\mathcal{I})$ such that $\mathcal{C}=\mathcal{C}(M)$.
\end{theorem}

According to the above theorem, we see a matroid can be determined by its circuits.
\begin{defn}(Normal matroid~\cite{Lai01Matroid})
\label{definition10}
Let $M=(U$, $\mathcal{I})$ be a matroid. If $\bigcup \mathcal{I}=U$, then we call
$M$ is a normal matroid.
\end{defn}

\begin{example}(Continued from Example \ref{example1})
\label{example3}
It is easy to prove the following results:\\
\begin{center}
$\bigcup\mathcal{I}_{1}=\{a$, $b$, $c\}=U$,\\
$\bigcup\mathcal{I}_{2}=\{a$, $b$, $c\}=U$.
\end{center}
Hence, the matroids $(U$, $\mathcal{I}_{1})$ and $(U$, $\mathcal{I}_{2})$ are two normal
matroids.
\end{example}

\begin{defn}(Union of matroids~\cite{Lai01Matroid})
\label{definition12}
Let $M_{1}=(U$, $\mathcal{I}_{1})$ and $M_{2}=(U$, $\mathcal{I}_{2})$ be two matroids.
Then we define the union of $M_{1}$ and $M_{2}$ as follows:\\
$M_{1}+M_{2}=(U$, $\mathcal{I}_{1}+\mathcal{I}_{2})$, where $\mathcal{I}_{1}+\mathcal{I}_{2}=\{I_{1}\bigcup I_{2}\mid I_{1}\in \mathcal{I}_{1}$, $I_{2}\in \mathcal{I}_{2}\}$.
\end{defn}
\begin{theorem}\cite{Lai01Matroid}
\label{theorem4}
If $M_{1}$ and $M_{2}$ are two matroids on $U$, then $M_{1}+M_{2}$ is a matroid.
\end{theorem}

\begin{example}(Continued from Example \ref{example1})
\label{example11}
According to Example \ref{example1} and Definition \ref{definition12}, it is easy
to compute
\begin{center}
$\mathcal{I}_{1}+\mathcal{I}_{2}=\{\emptyset$, $\{a\}$, $\{b\}$, $\{c\}$, $\{a$, $b\}$, $\{a$, $c\}$, $\{b$, $c\}$, $\{a$, $b$, $c\}\}$.
\end{center}
Clearly, $\mathcal{I}_{1}+\mathcal{I}_{2}$ satisfies (I1), (I2) and (I3) of Definition~\ref{definition7}, i.e., $(U$, $\mathcal{I}_{1}+\mathcal{I}_{2})$ is a matroid.
\end{example}
\section{Matroidal structure of generalized rough sets}
\label{section3}
In this section, we establish a matroidal structure for generalized rough sets based on a symmetric and
transitive relation, and study some properties of the matroidal structure. In~\cite{LiuZhuZhang12Relationshipbetween,WangZhuZhuMin12Matroidal,WangZhu12Quantitative,WangZhu11Matroidal}
some approaches to generate matroidal structures of Pawlak's rough sets and covering-based rough sets were proposed, respectively. Now, we propose an approach to generate a matroidal structure of generalized rough sets based on symmetric and transitive relations.
\begin{defn}
\label{definition11}
Let $R$ be a symmetric and transitive relation on $U$.
We can define $\mathcal{C}(R)$ as follows: for all $x$, $y\in U$,
\begin{center}
$\mathcal{C}(R)=\{\{x$, $y\}\mid x\neq y$, $y\in r_{R}(x)\}$.
\end{center}
\end{defn}
\begin{proposition}
\label{proposition1}
If $R$ is a symmetric and transitive relation on $U$, then $\mathcal{C}(R)$
satisfies the circuit axioms of matroids.
\end{proposition}
\begin{proof}
We need to prove $\mathcal{C}(R)$ satisfies (C1), (C2) and (C3) of Theorem \ref{theorem2}.
According to Definition \ref{definition11}, for all $C\in \mathcal{C}(R)$, $|C|=2$. Thus $\emptyset\notin \mathcal{C}(R)$ and (C1) holds. For all $C_{1}$, $C_{2}\in \mathcal{C}(R)$, $|C_{1}|=|C_{2}|=2$. If $C_{1}\subseteq C_{2}$,
then $C_{1}=C_{2}$ and (C2) holds. Let $C_{1}$, $C_{2}\in \mathcal{C}(R)$, $C_{1}\neq C_{2}$ and $e\in C_{1}\bigcap C_{2}$.
Without losing generality, let $C_{1}=\{x_{1}$, $e\}$ and $C_{2}=\{x_{2}$, $e\}$, where $x_{1}\neq x_{2}$. According to Definition \ref{definition11},
$e\in r_{R}(x_{1})$ and $e\in r_{R}(x_{2})$, since $R$ is symmetric and transitive, then $x_{1}\in r_{R}(x_{2})$. Suppose $C_{3}=\{x_{1}$, $x_{2}\}$, then $C_{3}\in \mathcal{C}(R)$ and $C_{3}\subseteq (C_{1}\bigcup C_{2})-\{e\}$. Thus $\mathcal{C}(R)$ satisfies (C3).
To sum up, $\mathcal{C}(R)$ satisfies the circuit axioms of matroids.
\end{proof}

Suppose there exists a matroid $M$ such that $\mathcal{C}(M)=\mathcal{C}(R)$, then we say $M(R)=(U$, $\mathcal{I}(R))$
is the induced matroid by $R$.

In fact, according to Theorem \ref{theorem3}, Definition \ref{definition11} and Proposition \ref{proposition1}, we find that a symmetric
and transitive relation on a universe determines a matroid. In other words, the matroidal
structure of generalized rough sets based on symmetric and transitive relations can be constructed.
\begin{example}
\label{example4}
Let $U=\{a$, $b$, $c\}$ and $R=\{(a$, $a)$, $(a$, $b)$, $(b$, $a)$, $(b$, $b)\}$ a symmetric and transitive
relation on $U$. According to Proposition 1, thus there exists a matroid $M(R)=(U$, $\mathcal{I}(R))$ induced by $R$, where
\begin{center}
$\mathcal{C}(R)=\{\{a$, $b\}\}$, \\
$\mathcal{I}(R)=\{\emptyset$, $\{a\}$, $\{b\}$, $\{c\}$, $\{b$, $c\}$, $\{a$, $c\}\}$.
\end{center}
\end{example}
\begin{proposition}
\label{proposition2}
If $R$ is a symmetric and transitive relation on $U$ and $M(R)=(U$, $\mathcal{I}(R))$
the matroid induced by $R$, then $M(R)$ is a normal matroid.
\end{proposition}
\begin{proof}
According to Proposition \ref{proposition1}, for all $C\in \mathcal{C}(R)$, then $|C|=2$. In other
words, for all $x$, then $\{x\}\in \mathcal{I}(R)$. Thus $\bigcup \mathcal{I}(R)=U$, i.e.,
$M(R)$ is a normal matroid.
\end{proof}
\begin{example}(Continued from Example \ref{example4})
\label{example5}
The matroid induced by $R$ is $M(R)=(U$, $\mathcal{I}(R))$, where $\mathcal{I}(R)=\{\emptyset$, $\{a\}$, $\{b\}$, $\{c\}$, $\{b$, $c\}$, $\{a$, $c\}\}$.
Since $\bigcup\mathcal{I}(R)=\{a$, $b$, $c\}=U$, thus $M(R)$ is a normal matroid.
\end{example}
\begin{proposition}
\label{proposition3}
Let $R$ be a symmetric and transitive relation on $U$ and $M(R)=(U$, $\mathcal{I}(R))$
the matroid induced by $R$. If $C_{1}$, $C_{2}\in \mathcal{C}(R)$, $e_{1}\in C_{1}-C_{2}$ ,
$e_{2}\in C_{2}-C_{1}$ and $C_{1}\bigcap C_{2}\neq \emptyset$, then there exists $C_{3}\in \mathcal{C}(R)$ such that
$e_{1}$, $e_{2}\in C_{3}\subseteq C_{1}\bigcup C_{2}$.
\end{proposition}
\begin{proof}
Let $C_{1}$, $C_{2}\in\mathcal{ C}(R)$. Without losing generality, let $C_{1}=\{x$, $e_{1}\}$ and $C_{2}=\{x$, $e_{2}\}$, where
$e_{1}\notin C_{2}$, $e_{2}\notin C_{2}$ and $x\neq e_{1}\neq e_{2}$. Since $R$ is a symmetric and transitive relation, thus $e_{2}\in r_{R}(e_{1})$. Suppose $C_{3}=\{e_{1}$, $e_{2}\}$, then $C_{3}\in \mathcal{C}(R)$ and $e_{1}$, $e_{2}\in C_{3}\subseteq C_{1}\bigcup C_{2}$.
\end{proof}
\begin{proposition}
\label{proposition4}
Let $R_{1}$ and $R_{2}$ be two symmetric and transitive relations on $U$. Let $M(R_{1})$, $M(R_{2})$ and
$M(R_{1}\bigcap R_{2})$ be the matroids induced by $R_{1}$, $R_{2}$ and $R_{1}\bigcap R_{2}$, respectively.
Then $\mathcal{I}(R_{1})\subseteq \mathcal{I}(R_{1}\bigcap R_{2})$ and $\mathcal{I}(R_{2})\subseteq \mathcal{I}(R_{1}\bigcap R_{2})$.
\end{proposition}
\begin{proof}
 Since $R_{1}\bigcap R_{2}\subseteq R_{1}$ and $R_{1}\bigcap R_{2}\subseteq R_{2}$, according to Definition \ref{definition11}, then
$\mathcal{C}(R_{1}\bigcap R_{2})\subseteq \mathcal{C}(R_{1})$ and $\mathcal{C}(R_{1}\bigcap R_{2})\subseteq \mathcal{C}(R_{2})$.
According to Definition \ref{definition8} and Definition \ref{definition9}, $\mathcal{C}(R_{1})\subseteq 2^{U}-\mathcal{I}(R_{1})$, $\mathcal{C}(R_{2})\subseteq 2^{U}-\mathcal{I}(R_{2})$ and
$\mathcal{C}(R_{1}\bigcap R_{2})\subseteq 2^{U}-\mathcal{I}(R_{1}\bigcap R_{2})$, therefore $\mathcal{I}(R_{1})\subseteq \mathcal{I}(R_{1}\bigcap R_{2})$ and $\mathcal{I}(R_{2})\subseteq \mathcal{I}(R_{1}\bigcap R_{2})$.
\end{proof}

\begin{example}
\label{example6}
Let $U=\{a$, $b$, $c$, $d\}$ be a universe.\\
Let $R_{1}=\{(a$, $a)$, $(a$, $b)$, $(b$, $a)$, $(b$, $b)$, $(a$, $c)$, $(c$, $a)$, $(b$, $c)$, $(c$, $b)$, $(c$, $c)\}$\\
and $R_{2}=\{(a$, $a)$, $(a$, $b)$, $(b$, $a)$, $(b$, $b)$, $(a$, $d)$, $(d$, $a)$, $(b$, $d)$, $(d$, $b)$, $(d$, $d)\}$ be two symmetric and
transitive relations on $U$, respectively.\\
Then $R_{1}\bigcap R_{2}=\{(a, a), (a, b), (b, a), (b, b)\}$ is a symmetric and transitive relation.
$M(R_{1})=(U$, $\mathcal{I}(R_{1}))$, $M(R_{2})=(U$, $\mathcal{I}(R_{2}))$ and $M(R_{1}\bigcap R_{2})=(U$, $\mathcal{I}(R_{1}\bigcap R_{2}))$
are induced by $R_{1}$, $R_{2}$ and $R_{1}\bigcap R_{2}$, respectively. Then\\
$\mathcal{I}(R_{1})=\{\emptyset$, $\{a\}$, $\{b\}$, $\{c\}$, $\{d\}$, $\{a$, $d\}$, $\{b$, $d\}$, $\{c$, $d\}\}$;\\
$\mathcal{I}(R_{2})=\{\emptyset$, $\{a\}$, $\{b\}$, $\{c\}$, $\{d\}$, $\{a$, $c\}$, $\{b$, $c\}$, $\{c$, $d\}\}$;\\
$\mathcal{I}(R_{1}\bigcap R_{2})=\{\emptyset$, $\{a\}$, $\{b\}$, $\{c\}$, $\{d\}$, $\{a$, $c\}$, $\{a$, $d\}$, $\{b$, $c\}$, $\{b$, $d\}$, $\{c$, $d\}$, $\{b$, $c$, $d\}$, $\{a$, $c$, $d\}\}$.\\
Hence, $\mathcal{I}(R_{1})\subseteq \mathcal{I}(R_{1}\bigcap R_{2})$ and $\mathcal{I}(R_{2})\subseteq \mathcal{I}(R_{1}\bigcap R_{2})$.
\end{example}

In fact, the intersection of symmetric and transitive relations is also a symmetric and transitive relation.
Hence, a matroid can be generated by the intersection of symmetric and transitive relations.

\begin{example}
\label{example7}
Let $U=\{a$, $b$, $c\}$ and $R_{1}=\{(a$, $a)$, $(a$, $b)$, $(b$, $a)$, $(b$, $b)\}$,
$R_{2}=\{(a$, $a)$, $(a$, $c)$, $(c$, $a)$, $(c$, $c)\}$ be two symmetric and transitive relations on $U$.
Then $M(R_{1})=(U$, $\mathcal{I}(R_{1}))$ and $M(R_{2})=(U$, $\mathcal{I}(R_{2}))$ are the two matroids induced
by $R_{1}$ and $R_{2}$, respectively. Where $\mathcal{I}(R_{1})=\{\emptyset$, $\{a\}$, $\{b\}$, $\{c\}$, $\{a$, $c\}$, $\{b$, $c\}\}$,
$\mathcal{I}(R_{2})=\{\emptyset$, $\{a\}$, $\{b\}$, $\{c\}$, $\{a$, $b\}$, $\{b$, $c\}\}$.
Hence, $M(R_{1})+M(R_{2})=(U$, $\mathcal{I}(R_{1})+\mathcal{I}(R_{2}))$ is a matroid and
$\mathcal{I}(R_{1})+\mathcal{I}(R_{2})=2^{U}$.
\end{example}

\begin{example}(Continued from Example \ref{example7})
Let $M(R_{1})\times M(R_{2})=(U$, $\mathcal{I}(R_{1})\bigcap\mathcal{I}(R_{2}))$. Then
$\mathcal{I}(R_{1})\bigcap\mathcal{I}(R_{2})=\{\emptyset$, $\{a\}$, $\{b\}$, $\{c\}$, $\{b$, $c\}\}$.
According to Definition \ref{definition7}, it does not satisfy (I3), thus ordered pair $(U$, $\mathcal{I}(R_{1})\bigcap\mathcal{I}(R_{2}))$ is not a matroid.
\end{example}

The above example indicates that the intersection of independent set families of two matroids may be not an independent set family of a matroid.

\section{Generalized rough sets based on matroids}
\label{section4}
The lower and upper approximation operators are the core concepts of rough sets. In \cite{Yao98Constructive} and \cite{ZhuWang06Binary},
the authors proved the existence and uniqueness of a certain binary relation for an algebraic operator with special properties. In this section, for a symmetric and transitive relation on a universe, we use the circuits of the matroid induced by the relation to represent the lower and upper approximations of the generalized rough sets.
\begin{proposition}
\label{proposition5}
Let $R$ be a symmetric and transitive relation on $U$ and $M(R)=(U$, $\mathcal{I}(R))$
the matroid induced by $R$. For all $X\in 2^{U}$,
\begin{center}
$\overline{R}(X)=Y_{1}\bigcup Y_{2}$,\\
$\underline{R}(X)=Y_{2}\bigcup Y_{3   }\bigcup Y_{4}$,
\end{center}
where $Y_{1}=\bigcup\{C\in\mathcal{C}(R)\mid C\bigcap X\neq\emptyset\}$, $Y_{2}=\{x\in X\mid x\in r_{R}(x)\}$, $Y_{3}=\bigcup\{C\in\mathcal{C}(R)\mid \forall x \forall y(x\in C\wedge\{x$, $y\}\in \mathcal{C}(R)\rightarrow \{x$, $y\}\subseteq X)\}$ and $Y_{4}=\{x\in X\mid r_{R}(x)=\emptyset\}$.
\end{proposition}
\begin{proof}
According to Definition \ref{definition1}, we need to prove
(1) $\{x\in U\mid r_{R}(x)\bigcap X\neq \emptyset\}=(\bigcup\{C\in \mathcal{C}(R)\mid C\bigcap X\neq\emptyset\})\bigcup \{x\in X\mid x\in r_{R}(x)\}$
and (2) $\{x\in U\mid r_{R}(x)\subseteq X\}=(\bigcup\{C\in\mathcal{C}(R)\mid \forall x \forall y(x\in C\wedge\{x$, $y\}\in \mathcal{C}(R)\rightarrow \{x$, $y\}\subseteq X)\})\bigcup \{x\in X \mid x\in r_{R}(x)\}\bigcup\{x\in X\mid r_{R}(x)=\emptyset\}$ hold.

(1) On one hand, if for any $x_{1}\in (\bigcup\{C\in \mathcal{C}(R)\mid C\bigcap X\neq\emptyset\})\bigcup \{x\in X\mid x\in r_{R}(x)\}$, then $x_{1}\in \{x\in X\mid x\in r_{R}(x)\}$ or $x_{1}\in \bigcup\{C\in \mathcal{C}(R)\mid C\bigcap X\neq\emptyset\}$. If $x_{1}\in \{x\in X\mid x\in r_{R}(x)\}$, then $x_{1}\in X$ and $x_{1}\in r_{R}(x_{1})$, i.e., $x_{1}\in \{x\in U\mid r_{R}(x)\bigcap X\neq \emptyset\}$; if $x_{1}\in \bigcup\{C\in \mathcal{C}(R)\mid C\bigcap X\neq\emptyset\}$, then there exists $x_{2}\neq x_{1}$ such that $\{x_{1}$, $x_{2}\}\in \mathcal{C}(R)$ and $\{x_{1}$, $x_{2}\}\bigcap X\neq \emptyset$. Hence, $r_{R}(x_{1})\bigcap X\neq \emptyset$. Thus $(\bigcup\{C\in \mathcal{C}(R)\mid C\bigcap X\neq\emptyset\})\bigcup \{x\in X\mid x\in r_{R}(x)\}\subseteq\{x\in U\mid r_{R}(x)\bigcap X\neq \emptyset\}$.
On the other hand, for any $x_{1}\in \{x\in U\mid r_{R}(x)\bigcap X\neq \emptyset\}$, there exists $x_{2}\in X$ such that $x_{2}\in r_{R}(x_{1})$. If $x_{1}\neq x_{2}$, then $\{x_{1}$, $x_{2}\}\in \mathcal{C}(R)$ and $x_{1}\in \bigcup\{C\in \mathcal{C}(R)\mid C\bigcap X\neq\emptyset\}$; if $x_{1}=x_{2}$, then $x_{1}\in \{x\in X\mid x\in r_{R}(x)\}$. Therefore, $\{x\in U\mid r_{R}(x)\bigcap X\neq \emptyset\}=(\bigcup\{C\in \mathcal{C}(R)\mid C\bigcap X\neq\emptyset\})\bigcup \{x\in X\mid x\in r_{R}(x)\}$ holds.

(2) On one hand, for any $x_{1}\in (\bigcup\{C\in\mathcal{C}(R)\mid \forall x \forall y(x\in C\wedge\{x$, $y\}\in \mathcal{C}(R)\rightarrow \{x$, $y\}\subseteq X)\})\bigcup \{x\in X \mid x\in r_{R}(x)\}\bigcup\{x\in X\mid r_{R}(x)=\emptyset\}$, $x_{1}\in \bigcup\{C\in\mathcal{C}(R)\mid \forall x \forall y(x\in C\wedge\{x$, $y\}\in \mathcal{C}(R)\rightarrow \{x$, $y\}\subseteq X)\}$ or $x_{1}\in \{x\in X \mid x\in r_{R}(x)\}$ or $x_{1}\in \{x\in X\mid r_{R}(x)=\emptyset\}$.
If $x_{1}\in \bigcup\{C\in\mathcal{C}(R)\mid \forall x \forall y(x\in C\wedge\{x$, $y\}\in \mathcal{C}(R)\rightarrow \{x$, $y\}\subseteq X)\}$ and for all $\{x_{1}$, $x_{2}\}\in \mathcal{C}(R)$, then $\{x_{1}$, $x_{2}\}\subseteq X$; if $x_{1}\in \{x\in X \mid x\in r_{R}(x)\}$, then
$x_{1}\in r_{R}(x_{1})$; if $x_{1}\in \{x\in X\mid r_{R}(x)=\emptyset\}$, then $r_{R}(x_{1})\subseteq X$ and $x_{1}\in \{x\in U\mid r_{R}(x)\subseteq X\}$. Thus $r_{R}(x_{1})\subseteq X$, i.e., $x_{1}\in \{x\in U\mid r_{R}(x)\subseteq X\}$. Therefore, $(\bigcup\{C\in\mathcal{C}(R)\mid \forall x \forall y(x\in C\wedge\{x$, $y\}\in \mathcal{C}(R)\rightarrow \{x$, $y\}\subseteq X)\})\bigcup \{x\in X \mid x\in r_{R}(x)\}\subseteq \{x\in U\mid r_{R}(x)\subseteq X\}$ holds. On the other hand, for any $x_{1}\in \{x\in U\mid r_{R}(x)\subseteq X\}$, $r_{R}(x_{1})\subseteq X$. For all $x_{2}\in r_{R}(x_{1})$, if $x_{1}=x_{2}$, then $x_{1}\in \{x\in X\mid x\in r_{R}(x)\}$; if $x_{1}\neq x_{2}$, then $\{x_{1}$, $x_{2}\}\in \mathcal{C}(R)$ and $\{x_{1}$, $x_{2}\}\subseteq X$, i.e., $x_{1}\in \bigcup\{C\in\mathcal{C}(R)\mid \forall x \forall y(x\in C\wedge\{x$, $y\}\in \mathcal{C}(R)\rightarrow \{x$, $y\}\subseteq X)\}$.
Therefore, $\{x\in U\mid r_{R}(x)\subseteq X\}=(\bigcup\{C\in\mathcal{C}(R)\mid \forall x \forall y(x\in C\wedge\{x$, $y\}\in \mathcal{C}(R)\rightarrow \{x$, $y\}\subseteq X)\})\bigcup \{x\in X \mid x\in r_{R}(x)\}$ holds.

To sum up, we have already finished the proof of this proposition.
\end{proof}

The above proposition presents that the lower and upper approximation operators of generalized rough sets
based on symmetric and transitive relations can  be described by the circuits of matroids. Therefore, we can compute
the approximation quality and roughness of generalized rough sets by the circuits of matroid.

\begin{example}
Let $U=\{a$, $b$, $c$, $d$, $e$, $f\}$ be a universe and $R=\{(a$ ,$a)$, $(a$, $b)$, $(b$, $a)$, $(b$, $b)$, $(a$, $d)$,
$(d$, $a)$, $(b$, $d)$, $(d$, $b)$, $(d$, $d)$, $(c$, $c)$, $(e$, $e)\}$ a symmetric and transitive relation on $U$. Suppose
$X_{1}=\{a$, $b$, $c$, $e$, $f\}$ and $X_{2}=\{a$, $c$, $d\}$, then the lower and upper approximations,
approximation quality and roughness of $X_{1}$ and $X_{2}$ could be computed as follows, respectively.

(1) According to Definition \ref{definition5} and Definition \ref{definition6}, since $r_{R}(a)=\{a$, $b$, $d\}$, $r_{R}(b)=\{a$, $b$, $d\}$,
$r_{R}(c)=\{c\}$, $r_{R}(d)=\{a$, $b$, $d\}$, $r_{R}(e)=\{e\}$, $r_{R}(f)=\emptyset$, then
$\underline{R}(X_{1})=\{c$, $e$, $f\}$, $\overline{R}(X_{1})=\{a$, $b$, $c$, $d$, $e\}$, $\alpha_{R}(X_{1})=\frac{|\underline{R}(X_{1})|}{|\overline{R}(X_{1})|}=0.6$, $\rho_{R}(X_{1})=1-\alpha_{R}(X_{1})=0.4$; ~$\underline{R}(X_{2})=\{c$, $f\}$, $\overline{R}(X_{2})=\{a$, $b$, $c$, $d\}$, $\alpha_{R}(X_{2})=\frac{|\underline{R}(X_{2})|}{|\overline{R}(X_{2})|}=0.5$, $\rho_{R}(X_{2})=1-\alpha_{R}(X_{2})=0.5$.

(2) According to Proposition \ref{proposition5} and Definition \ref{definition6}, since $\mathcal{C}(R)=\{\{a$, $b\}$, $\{a$, $d\}$, $\{b$, $d\}\}$, then
$\underline{R}(X_{1})=\{c$, $e$, $f\}$, $\overline{R}(X_{1})=\{a$, $b$, $c$, $d$, $e\}$, $\alpha_{R}(X_{1})=\frac{|\underline{R}(X_{1})|}{|\overline{R}(X_{1})|}=0.6$, $\rho_{R}(X_{1})=1-\alpha_{R}(X_{1})=0.4$; ~$\underline{R}(X_{2})=\{c$, $f\}$, $\overline{R}(X_{2})=\{a$, $b$, $c$, $d\}$, $\alpha_{R}(X_{2})=\frac{|\underline{R}(X_{2})|}{|\overline{R}(X_{2})|}=0.5$, $\rho_{R}(X_{2})=1-\alpha_{R}(X_{2})=0.5$.
\end{example}

In fact, according to Proposition \ref{proposition5}, we know the lower and upper approximation operators
of generalized rough sets based on symmetric and transitive relations can be described by the circuits of
matroids induced by symmetric and transitive relations. Thus according to Theorem \ref{theorem1}, we can
describe some properties of a symmetric and transitive relation by the circuits of the matroid induced by this relation.

\begin{defn}
\label{definition13}
Let $M=(U$, $\mathcal{I})$ be a matroid. We can define a relation $R(M)$ as follows: for all $x$, $y\in U$
\begin{center}
$xR(M)y\Longleftrightarrow \{x, y\}\in \mathcal{C}(M)$ or $x=y$.
\end{center}
\end{defn}
\begin{proposition}
\label{proposition6}
If $M=(U$, $\mathcal{I})$ is a matroid, then the relation $R(M)$ is symmetric and transitive.
\end{proposition}

\begin{proof}
According to Definition \ref{definition13}, for all $C=\{x_{1}$, $y_{1}\}\in \mathcal{C}(M)$, thus $(x_{1}$, $y_{1})\in R$.
Since $\{x_{1}$, $y_{1}\}=\{y_{1}$, $x_{1}\}$, then $(y_{1}$, $x_{1})\in R$ and symmetry of $R$ holds. For all $\{x_{1}$, $y_{1}\}\in \mathcal{C}(M)$ and $\{y_{1}$, $z_{1}\}\in \mathcal{C}(M)$, according to the (C3) of Theorem \ref{theorem2}, thus $(x_{1}$, $z_{1})\in R$. Therefore, the transitivity of $R$ satisfies.
\end{proof}

In fact, the above proposition shows how to construct a symmetric and transitive relation by a matroid.

\begin{example}
Let $U=\{a$, $b$, $c$, $d$, $e$, $f\}$ be a universe and $M=(U$, $\mathcal{I})$ a matroid on $U$.
If $\mathcal{C}(M)=\{\{a$, $b\}$, $\{a$, $c\}$, $\{a$, $e\}$, $\{b$, $c\}$, $\{c$, $f\}$, $\{e$, $f\} \}$,
according to Proposition \ref{proposition7}, then the symmetric and transitive relation $R(M)$ induced by $M$ is as follows:
\begin{center}
$R(M)=\{(a$, $a)$, $(b$, $b)$, $(c$, $c)$, $(e$, $e)$, $(f$, $f)$, $(a$, $b)$, $(b$, $a)$, $(a$, $c)$, $(c$, $a)$, $(a$, $e)$, $(e$, $a)$, $(b$, $c)$, $(c$, $b)$, $(c$, $f)$, $(f$, $c)$, $(e$, $f)$, $(f$, $e)\}$.
\end{center}
\end{example}

In Section \ref{section3}, we have already proved the union of matroids is a matroid. Now, in following proposition we will
explore the relationships between the symmetric and transitive relation generated by the union of two matroids and the two relations generated
by these two matroids, respectively.

\begin{proposition}
\label{proposition7}
Let $R_{1}$ and $R_{2}$ be two symmetric and transitive relations on $U$.
If the matriods $M(R_{1})$ and $M(R_{2})$ were induced by $R_{1}$ and $R_{2}$, respectively, then $R(M(R_{1})+M(R_{2}))$
is an empty relation.
\end{proposition}

\begin{proof}
According to Definition \ref{definition12} and Definition \ref{definition13}, it is
easy to prove $\mathcal{C}(M(R_{1})+M(R_{2}))=\emptyset$. Hence, $R(M(R_{1})+M(R_{2}))$
is an empty relation.
\end{proof}

\begin{example}(Continued from Example \ref{example7})
According to Example \ref{example7}, we have already known $\mathcal{I}(M(R_{1})+M(R_{2}))=2^{U}$.
Therefore, $\mathcal{C}(M(R_{1})+M(R_{2}))=\emptyset$ and $R(M(R_{1})+M(R_{2}))$
is a empty relation.
\end{example}

\section{Conclusions}\label{section5}
In this paper we construct the matroidal structure of a symmetric and transitive relation
on a nonempty and finite set. Firstly, we use properties and results of the generalized rough
sets to study the properties of the matroid induced by a symmetric and transitive relation. Secondly,
we represent the lower approximation operator and upper approximation operator by the circuits of the
matroid induced by a symmetric and transitive relation. Finally, a symmetric and transitive relation can be generated
by a matroid with some special properties.
\section*{Acknowledgements}\label{section: acknowledgements}
This work is supported in part by the National Natural Science Foundation
of China under Grant No. 61170128, the Natural Science Foundation of Fujian
Province, China, under Grant Nos. 2011J01374 and 2012J01294, and the Science
and Technology Key Project of Fujian Province, China, under Grant No.
2012H0043.

\end{document}